\documentclass[sigconf]{acmart}

\usepackage{amsmath}
\usepackage{mathtools}
\usepackage{multirow, booktabs}
\usepackage{pdfpages}
\usepackage{array}
\usepackage{amsthm}
\usepackage{graphicx, subcaption}
\usepackage{hyperref}

\newcommand{\calG}{\mathcal{G}}
\newcommand{\calV}{\mathcal{V}}
\newcommand{\calE}{\mathcal{E}}
\newcommand{\bszero}{\mathbf{0}}
\newcommand{\bse}{\mathbf{e}}

\newtheorem{theorem}{Theorem}[section]

\begin{document}

\title{Fake News Quick Detection on Dynamic Heterogeneous Information Networks}

\author{Jin Ho Go, Alina Sari, Jiaojiao Jiang, Shuiqiao Yang, Sanjay Jha}
\email{{jinho.go, a.sari, jiaojiao.jiang, shuiqiao.yang, sanjay.jha}@unsw.edu.au}
\affiliation{%
  \institution{School of Computer Science, University of New South Wales}
  \city{Sydney}
  \state{NSW}
  \country{Australia}
  \postcode{2052}
}




\begin{abstract}
The spread of fake news has caused great harm to society in recent years. So the quick detection of fake news has become an important task. Some current detection methods often model news articles and other related components as a \textit{static} heterogeneous information network (HIN) and use expensive message-passing algorithms. However, in the real-world, quickly identifying fake news is of great significance and the network may vary over time in terms of dynamic nodes and edges. Therefore, in this paper, we propose a novel Dynamic Heterogeneous Graph Neural Network (DHGNN) for fake news quick detection. More specifically, we first implement BERT and fine-tuned BERT to get a semantic representation of the news article contents and author profiles and convert it into graph data. Then, we construct the heterogeneous news-author graph to reflect contextual information and relationships. Additionally, we adapt ideas from personalized PageRank propagation and dynamic propagation to heterogeneous networks in order to reduce the time complexity of back-propagating through many nodes during training. Experiments on three real-world fake news datasets show that DHGNN can outperform other GNN-based models in terms of both effectiveness and efficiency.

\end{abstract}


\maketitle

\section{Introduction}
\label{sec:intro}
With the advancement of technology and fast dissemination of information, the issue of fake news has also escalated, e.g., the 2016 U.S. presidential election. The discredit caused by misinformation results in unrest and confused public opinions regarding essential subjects, which could endanger the stability and harmony in society\cite{MM21}.
\begin{figure}
    \centering
    \includegraphics[scale=0.42]{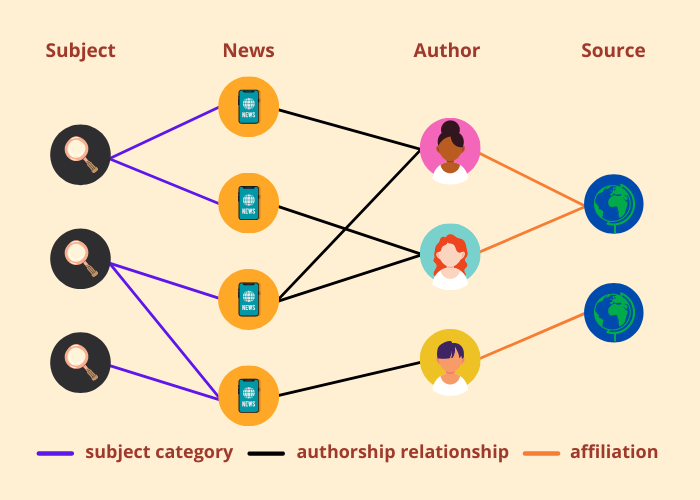}
    \caption{An illustrative example of a heterogeneous information network with four types of nodes (i.e., subject, news, author and source) and three types of links (i.e., subject-news, news-author, and author-source).}
    \label{fig:hetero_news_graph_og}
\end{figure}

In the past, we require expert fact-checkers to ensure the veracity of a piece of news. However, this process is time-consuming and expensive, which is infeasible in a globalized world with millions of online social media users posting and sharing news every day. Therefore, there have been some serious studies on identifying fake news using an automated system with the help of artificial intelligence. The early efforts include the application machine learning models, such as support vector machines (SVMs), decisions trees, naive bayes and random forest, in classifying news in a data set. The results often vary depending on the domain and data set. If interested, \cite{FakeNewsDetSurvey} summarizes the performance of some traditional machine learning models on the fake news detection task. 

Deep learning methods such as convolutional neural networks (CNNs) and recurrent neural networks (RNNs) are also employed in the news classification task. The general method begins by extracting the news content and converting it into a text embedding, which is then passed into a convolutional and/or recurrent networks, whose output indicates whether the news is fake. An early example of such method is using CNN for sentence classification, as shown in \cite{CNNSentenceClassification}. 

More recently, there has been an extension of CNN to non-Euclidean graph-structured data known as graph neural network (GNN), pioneered by the introduction of graph convolutional network in e.g., \cite{GCN}. This generalisation allows for many real-world applications such as social networks, bio-molecular graphs, automated traffic systems and mis(dis)information detection, whose data is often large, irregularly-shaped and heterogeneous, to be represented as a graph. 
Figure \ref{fig:hetero_news_graph_og} shows an example of a heterogeneous news graph where the nodes consist of news articles, creators and subjects, and the edges consist of news-creator and news-subject relationships. This allows for more contextual information to be considered in the model, which can potentially improve the quality in predicting misinformation than homogeneous graph\cite{info_cred}. There are some recent studies on employing heterogeneous networks for fake news detection, see e.g., \cite{FANG, HGAT}.

A different approach to this problem is by observing the dissemination of the news after they are released. This can be achieved by monitoring the news propagation through users who share the news with their friends/followers on social media. The principle of this method is that fake news tends to propagate faster and to a broader population compared to real news \cite{news_spread, news_prop}. However, this method requires the news to be spread to some people before it is able to make a prediction, in which the damage has been done. 

To handle this issue, we propose a novel dynamic heterogeneous attention network, namely Dynamic Heterogeneous Graph Neural Network (DHGNN). DHGNN aims to detect fake news in a timely manner when an incoming piece of news is just released. Although there have been some studies on early fake news detection, see e.g., \cite{early_det}, there are not many kinds of research that considered timely fake news detection with the heterogeneity of the news graph data. For DHGNN, we first implement BERT and finetuned BERT to get a semantic representation of the news article contents and author profiles and convert it into graph data. Then, we construct the heterogeneous news-author graph to reflect contextual information and relationships. Additionally, we adapt ideas from Personalized PageRank propagation in \cite{PPNP} and dynamic propagation in \cite{SDG} to heterogeneous networks in order to reduce the time complexity of back-propagating through many nodes during training.

The structure of this paper is as follows. We introduce related works, which are the state-of-art on fake news detection, in Section \ref{sec:related}. Then, the proposed model is described in Section \ref{sec:prop_model} by going through an overview of the network architecture before focusing on specific parts of the process, such as the data extraction section and the propagation scheme. We compare the training time and prediction quality of our model with some existing state-of-the-art homogeneous and heterogeneous GNNs in Section \ref{sec:experiments}, before finally concluding our studies and findings along with some prospects in the research field in Section \ref{sec:conclusion}.


\section{Related Work}
\label{sec:related}
In this section, we present state-of-art Heterogeneous Graph Neural Networks and Dynamic Graph Neural Networks for fake news detection. Also, we examine the limitation of two different approaches.

\subsection{Heterogeneous Graph Neural Networks for fake news detection}
Heterogeneous Information Networks(HINs) are more suitable for representing complex relationships in the real world, compared
to homogeneous network which consists of only one type of node and edge. In particular, HINs are useful in fake news detection because they can express the social context that occurs through the interaction between users and news on social media, which are widely used worldwide due to the development of the mobile Internet.

To extract abundant social interconnection between entities on social networks, Wang et al. \cite{FANG} suggest graph representation learning for detecting fake news, instead of transductive learning. Unlike transductive embedding frameworks, which rely on neighbourhood information only, GraphSage\cite{hamilton2017inductive} makes use of auxiliary node attribute, as well as proximity sampling, and is implemented to represent HINs of relationship between entities on social platform. By leveraging the underlying social context, this inductive graph learning model improves the quality of representation with minimum training data. 

FAKEDETECTOR\cite{zhang2020fakedetector} is designed to handle not only the credibility of the news but also the author and subjects at the same time. Explicit and latent features extracted through representation feature learning are used to model the correlation of entities in a gated diffusive unit(GDU)\cite{chung2014empirical}. GDU, which can handle multiple inputs simultaneously, can effectively form HINs of social networks.

AA-HGNN\cite{ren2020adversarial} is a heterogeneous GNN which uses a hierarchical attention mechanism for node representation learning.This model specializes in detecting fake news at publication using adversarial active learning framework. Similar to active learning, adversarial active learning employs a classifier and a selector, which work hand-in-hand in improving the quality of predicted labels. The addition of the adversarial learning concept from General Adversarial Network (GAN)\cite{goodfellow2014generative} helps with the limited training data in real-world.

Heterogeneous graph neural networks reflect more prosperous relationships between different types of nodes through HINs. However, since they are all based on a static graph, it is difficult to track the relationship between nodes over time. Thus, they are not very suitable for real-time fake news detection.

\subsection{Dynamic Graph Neural Networks for fake news detection}

A static graph has a fixed structure(nodes and edges) and parameters\cite{monical2014static}. On the other hand, a dynamic graph is a graph that reflects changes caused by time or changes in environmental elements such as temperature\cite{monical2014static}. Dynamic graphs are more suitable for modelling real-world problems that are time- sensitive or prompt to changes\cite{song2021temporally}. However, compared to static graph, it is complicated to handle, so many studies are being conducted to deal with dynamic graph.

Song et al.\cite{song2021temporally} indicates the limitations of GNNs models based on static graphs. A continuous-time dynamic GNN, called TGNF, has been proposed to overcome this limitation. It combines the the temporal graph attention neural networks(TGAT) model with a graph convolutional layer to capture the evolution pattern of news propagation in continuous time. However, as a disadvantage, the model took a relatively long execution time compared to other comparative models.

SDG is a GNN that is dynamic and homogeneous\cite{fu2021sdg}. It uses a dynamic propagation scheme because the influence of other nodes on a particular node in a k-layer GNN model is essentially a k-step random walk from that particular node. As the graph structure changes, it was proven that the push-out algorithm used to update the model converges to the stationary distribution of the updated graph. Hence, SDG can find the new dynamic propagation matrix without an expensive cost using the push-out algorithm.

Dynamic GNN models are handy for capturing real-world fluid relationship shifting between nodes over time. However, in most cases, the dynamic structure is captured based on a homogeneous graph, so there are limitations in expressing the various relationships generated by different types of nodes.

\section{Problem Statement} 
\label{sec:problem}

\subsection{Terminology Definition}
We describe the the definition of Heterogeneous Information Networks(HINs) on our model as follows:

\textit{DEFINITION 1:} (Heterogeneous Information Networks(HINs)): The news-author based heterogeneous information networks(HINs) is denoted as $G = (V,E)$. $V$ represents the node-set $V = A\cup N \cup S \cup R$, which is $A$ means the set of authors, $N$ means the set of news articles, $S$ means the set of subjects and $R$ represents the set of sources. These three types of nodes are covered in detail later. $E$ denotes the relationships among author and news nodes. There are three types of relationships that are considered: the authorship relation $E_{an}$ between author and news nodes, the affiliation relation $E_{aa}$ between authors, and the subject category relation $E_{nn}$ between news articles. 

\textit{DEFINITION 2:} The word graph $\calG = (\calV, \calE)$ represents each news article and author with $\calV$ denoting each word in the article content and $\calE$ denoting the proximity of one word from another. More specifically, given a  positive integer $q$, if word $\calV_1$ lies within $q$ units from $\calV_2$, then $\calE_{1,2}$ = 1.

Author refers to people who post news article on social media. To define author in formal method, we are described as follows:

\textit{DEFINITION 3:} (Authors): It is the set of authors and represented as $A =\{a_{1}, a_{2}, a_{3}, ... , a_{l}\}$. Each author entity is $a_{i}\in A$, which includes the profile textual.

News article is defined as text content posted by author. To define news article in formal method, we are described as follows:

\textit{DEFINITION 4:}(News Articles): It is the set of news articles and represented as $N =\{n_{1}, n_{2}, n_{3}, ... , n_{m}\}$. Each news article entity is $n_{i}\in N$, which includes news content textual. The reliability of $n_{i}$ is labelled from the label-set $Y = \{Real, Fake\}$.

Subject denotes common topic for news article. To define news article in formal method, we are described as follows:

\textit{DEFINITION 5:} (Subjects): It is the set of subjects and represented as $S =\{s_{1}, s_{2}, s_{3}, ... , s_{n}\}$. Each subject entity is $s_{i}\in S$, which contains topic area textual.

Source means the platform used to post news article.To define source in formal method, we are described as follows:

\textit{DEFINITION 6:} (Sources): It is the set of source and represented as $R =\{r_{1}, r_{2}, r_{3}, ... , r_{o}\}$. Each subject entity is $r_{i}\in R$, which contains name of source textual.

\subsection{Problem Formulation}

The problem we are interested in is to learn a classification function $f:N \rightarrow Y$ on predicting the veracity of an incoming news article $n_{i}$ in a timely manner by modelling the task as a heterogeneous graph $G = (V,E)$ classification problem. The heterogeneity of the graph takes into account contextual information such as author, subject and source for each news article and the heterogeneous relationships among the graph nodes (Figure \ref{fig:hetero_news_graph_og}), whereas the timeliness of the prediction comes from the use of dynamic propagation, which will be described in Section \ref{sec:prop}. To address the above issues, we introduce the \textbf{D}ynamic \textbf{H}eterogeneous \textbf{G}raph \textbf{N}eural \textbf{N}etwork (\textbf{DHGNN}), a novel GNN model which has the following features:
\begin{enumerate}
    \item the use of BERT and finetuned BERT to get a good semantic representation of the news article contents and author profiles.
    \item the separation of the model classifier and its propagation scheme allows a wide range of nodes to be considered with a low number of layers.
    \item dynamic propagation scheme, which is derived from a variant of the personalized pagerank algorithm, allows for timely prediction to be made.
\end{enumerate}
The new contributions we made include:
\begin{enumerate}
    \item The convergence and existence proof of two-hop and mixed (one-hop and two-hop) dynamic propagation schemes.
    \item The application of two-hop and mixed dynamic propagation schemes to bipartite and heterogeneous agnostic models.
\end{enumerate}

\section{Proposed Model and Hypothesis} \label{sec:prop_model}

\subsection{Overview: Proposed Model}

We introduce the \textbf{D}ynamic \textbf{H}eterogeneous \textbf{G}raph \textbf{N}eural \textbf{N}etwork (\textbf{DHGNN}), a novel GNN model.

In addition to the heterogeneous model, we investigate a bipartite model consisting of only the authorship relation between news and author, which, similar to the heterogeneous model, is called the \textbf{D}ynamic \textbf{B}ipartite \textbf{G}raph \textbf{N}eural \textbf{N}etwork (\textbf{DBGNN}). This is done to observe the importance of subject and source information to our model's prediction.



Figure \ref{fig:net_arch} shows the overall network architecture and process from the data extraction of the raw news articles to obtaining prediction through our models. Firstly, each incoming news article is first converted to word graphs $\calG$ of each news content and author profile. Then, the data extraction module processes each word graph into a graph embedding vector using BERT model and personalized pagerank algorithm to take into account the word semantic and the influence of words that are close to each other in a word graph. Secondly, the news-author mapping can be obtained from the incoming news article. In addition, by considering the news subject category and author affiliation relationship among the news and author nodes, we can get the news-news and author-author mappings if two news have the same subject and two authors have the same publication source, respectively. Finally, the graph embedding vectors and the mappings are fed into either the DHGNN or DBGNN model, which predicts whether the input news is fake. Each step in the network architecture will be discussed in detail in the subsequent sections.

\begin{figure*}[t!]
    \centering
    \includegraphics[scale=0.58]{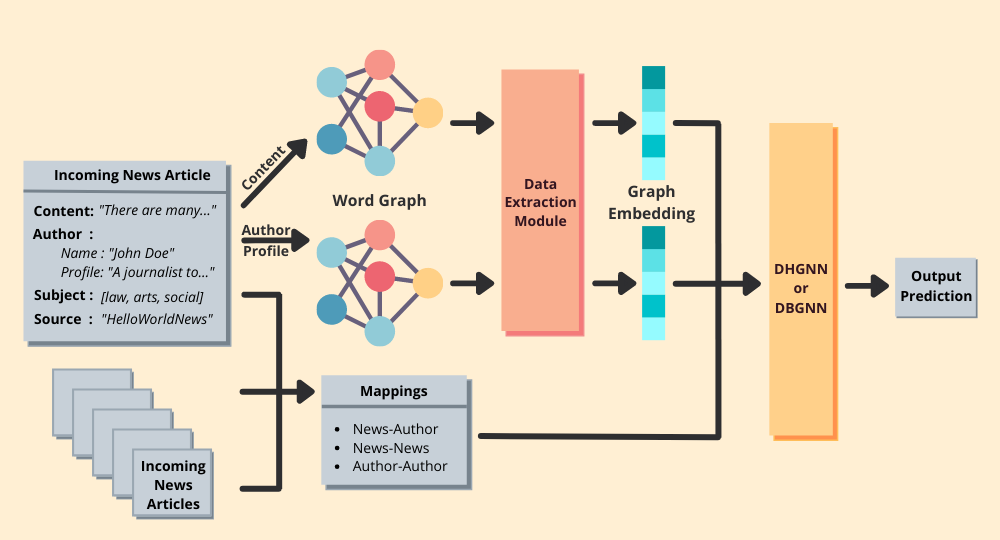}
    \caption{The overall architecture of the proposed method.}
    \label{fig:net_arch}
\end{figure*}

\subsection{Detailed Model Architecture}
\subsubsection{Data Extraction Method} \label{sec:data_extract}
As the information provided in the data sets and their structures vary, we need to preprocess them to extract some important information that is relevant to our proposed model. Figure \ref{fig:data_extract} shows the data extraction module.
There are three main steps in extracting the data: (i) selecting and cleaning the data, (ii) evaluating the word semantics using BERT \cite{BERT}, and (iii) converting the result into a graph data input for the model.

\subsubsection*{Selecting and Cleaning the data}
In general, the data sets used in our experiments are quite clean. However, as they are not designed specific for our purpose, we need to firstly select the relevant information. More specifically, the following fields are retrieved (if exist):
\begin{enumerate}
    \setlength\itemsep{0em}
    \item News: \textsc{id}, content, author(s), subject(s), label.
    \item Author: \textsc{id}, name, profile, source/affiliation, label.
    \item Source: \textsc{id}, name.
    \item Subject: \textsc{id}, name.
\end{enumerate}
\begin{figure*}[t!]
    \centering
    \includegraphics[scale=0.35]{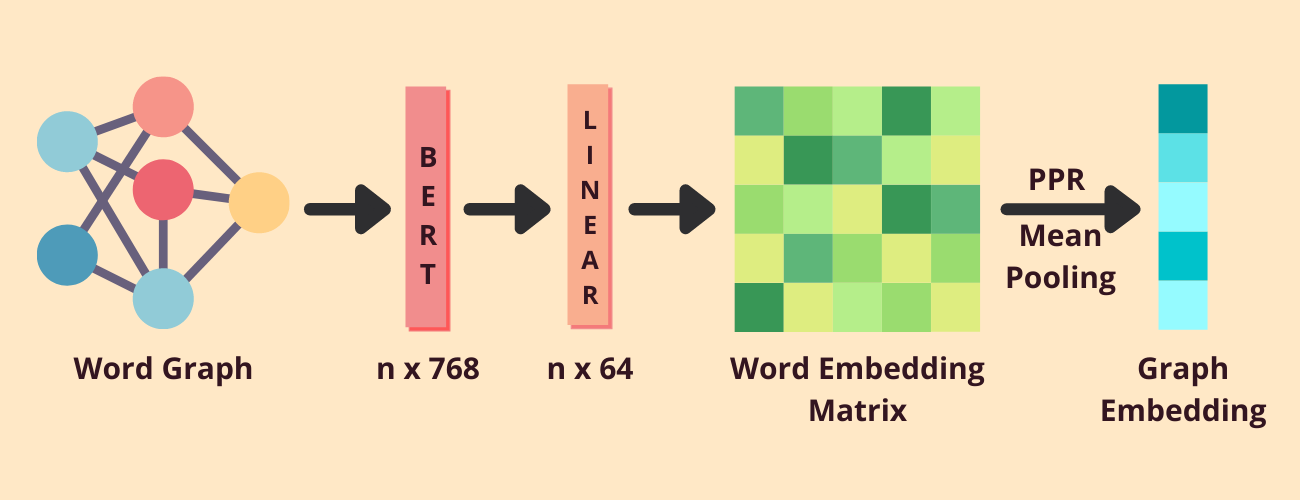}
    \caption{Data extraction module}
    \label{fig:data_extract}
\end{figure*}
The following rules and conditions are applied to each data set:
\begin{itemize}
    \setlength\itemsep{0em}
    \item The news content and author profile must be in English. Texts in other languages are removed from the data set.
    \item If the data does not come with an \textsc{id}, we will assign a unique positive integer as \textsc{id}.
    \item If there is no news content or label, the data point is filtered out.
    \item If there is/are no subject(s), the field is set to empty.
    \item If author profile is not given, the field is set to the author name.
    \item If there is no author or source, author name and profile are set to be unknown.
    \item If there is no author but there is a source, we set the author to be the source.
    \item The author label is computed as the average of the authored news labels.
    \item Non-binary news classification is converted to binary classification (real/fake).
\item Punctuation is removed from news content and author profile.
\end{itemize}

\subsubsection*{Finetuning BERT}
BERT (Bidirectional Encoder Representation from Transformers) is a language representation model which uses the idea from Transformer \cite{transformer} and achieves state-of-the-art results on various NLP tasks, including text classification. The original source code by Google Research team is made available \href{https://github.com/google-research/bert}{here}, which is adapted by \href{https://huggingface.co/}{Hugging Face} in their \href{https://huggingface.co/transformers/index.html}{Transformers library} for PyTorch, TensorFlow, and Jax. 

For our purpose, we use the pretrained base BERT model and tokenizer for uncased English texts from Hugging Face Transformers. The model consists of 12 layers, 768 hidden channels and 12-headed attention mechanism. However, the maximum input sequence length is 512, which is too small, considering some news articles contain thousands of words. Furthermore, the pretrained model is trained on Wikipedia pages, which is not specific to our problem. Hence, we want to finetune the pretrained model with our data sets.
For the finetuning process of the pretrained BERT model, in general, we train the pretrained BERT model and pool the result through a final linear (dense) layer, which yields the final output prediction. The idea follows the suggestion given by BERT's first author on this \href{https://github.com/google-research/bert/issues/27}{Github issue}.

\subsubsection*{Getting BERT word token embedding}
The process of obtaining tokenized sequence embedding with a BERT model is outlined as follows:
\begin{enumerate}
    \setlength\itemsep{0em}
    \item Each cleaned textual description, i.e., news content and author profile, is tokenized using the base BERT model tokenizer. If the resulting tokenized sequence length is longer than 512, it is split into tokenized sequences of length 512. If the sequence is shorter than 512, it is padded with zeros. Note that we need to ensure that each split tokenized sequence starts and ends with tokens ('[CLS]', 101) and ('[SEP]', 102), respectively.
    
    \item The token-index mapping is also saved so that we can construct each word embedding for our feature matrices. This is due to how each word is tokenized. 
    Note that some words such as 'gruesome' and 'closeby' are split into multiple tokens. Detailed information on various BERT tokenizers is available \href{https://huggingface.co/transformers/tokenizer_summary.html}{here}.
    \item The pretrained BERT model is loaded from Transformers and set to evaluation mode. Alternatively, we can finetune the pretrained BERT model as described in the previous section.
    \item There are different methods in computing the final tokenized sequence embedding using BERT. Comparisons of some of these methods are reported in \cite[Table 7]{BERT}. In our experiment, the embedding vector of eachtokenized sequence is computed by adding the output of the last four hidden layers of the pretrained BERT model.
\end{enumerate}
At the end of the last step, we have an embedding matrix along with a token-index mapping for each tokenized sequence, which can be converted into a feature matrix for each sequence. Note that we apply this process to both news content and author profile.

There are of course other methods we can use to get text embedding such as Word2Vec \cite{Word2Vec}, GloVe \cite{GloVe} and TF-IDF \cite{TFIDF}. However, they are context independent, which means that each polysemous word is represented by the same vector. For example, the word 'fly', regardless whether it is a noun or verb will be represented by the same vector. As context is important in the news classification task, we choose BERT to generate each word embedding in a text.

\subsubsection*{Conversion to graph data}
To use our data in a GNN model, we need to convert it into graph input data, as described in Section \ref{sec:problem}. This includes an adjacency matrix and a feature matrix for each word graph $\calG$, mappings among authors and news articles to represent the edges in graph $G$, and a list of labels corresponding to each word graph.\\
The construction of the feature matrix can be done in the following steps:
\begin{enumerate}
    \setlength\itemsep{0em}
    \item Get the tokenized sequence embedding matrix and token-index mapping of each sequence.
    \item If each word in the sequence is represented by one token, we take the token embedding vector as the word embedding vector. However, if each word is represented by multiple tokens, we take the average of the token embedding vectors as the word embedding vector.
    \item The feature matrix of each sequence is obtained by concatenating the word embedding vectors.
\end{enumerate}
The construction of the adjacency matrix can be done in the following steps:
\begin{enumerate}
    \setlength\itemsep{0em}
    \item As outlined in Section \ref{sec:problem}, the edge $\epsilon$ denotes the proximity of words in the graph, represented by the window size $q$. The value used in our experiments is $q = 3$.
    \item Since each word in a sequence is sequentially process for the feature matrix, the adjacency matrix is a sparse matrix with up to four diagonals with all ones. If the sequence is empty, an empty matrix is returned.
\end{enumerate}
There are three different mappings on the heterogeneous graph $G$: 
\begin{itemize}
    \setlength\itemsep{0em}
    \item author-author: an edge exists between two authors if their news articles are published on the same media channel.
    \item news-news: an edge exists between two news articles if they are on at least one same topic.
    \item author-news: an edge exists between author and news article if the news is written by the author.
\end{itemize}
The feature and adjacency matrices, along with the labels and mappings among news and authors, are then saved into a file, to be used in the model training and evaluation.
Note that in this experiment, we have multiple instances of the same word in the word graph. This is due to the fact that many English words are polysemous and using BERT, we are able to differentiate the word semantics. Hence, rather than pooling the result, we preserve each word embedding for our classification task.

\subsubsection{Model Implementation}

\begin{figure*}
    \centering
    \includegraphics[scale=0.6]{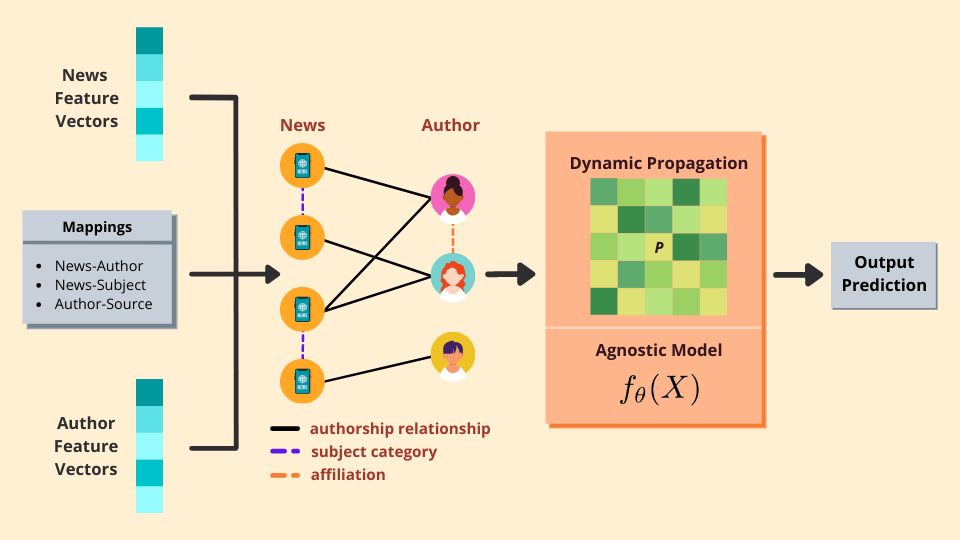}
    \caption{Overview of DHGNN and DBGNN}
    \label{fig:DHGNN_DBGNN}
\end{figure*}

Figure \ref{fig:DHGNN_DBGNN} shows the architecture of DBGNN and DHGNN. The model consists of a classifier and a propagation method, which are implemented separately. This is due to the application of dynamic propagation to replace the usual message passing in GNN. The outline of the network is shown in the following sections.

\subsubsection*{Getting the news-author graph data}
To construct the news-author graph, we firstly combine the word embedding vectors in each sequence. The idea is to compute the influence of each neighbouring word in the sequence using a personalized pagerank algorithm. Then, the aggregate embedding vector of each text is just the weighted average of the neighbouring words' embedding vectors. This weighted average serves as the feature vector of the sequence and it is applied to every news content and author profile. The next step is to construct the news-author graph. The feature matrix is obtained by appending the feature vectors we obtained in the previous step. The adjacency matrix is constructed using the news-author mapping that we get from the data extraction phase whereas the news and author labels can be used as they are.



\subsubsection*{Agnostic Model/Classifier $f_\theta(X)$}
The classifier used is a vanilla Graph Convolutional Network (GCN) without the usual message-passing propagation, which was initially implemented in \cite{PPNP}. The timeliness of the model is achieved by replacing the usual message passing propagation mechanism with dynamic propagation, which will be described in detail in Section \ref{sec:prop}. 

\subsubsection*{$K$-Fold Cross Validation}
The $K$-Fold Cross Validation is a common method to robustly test the prediction quality of a model. The data set is divided into $K$ even (or almost even) partitions, and then in each iteration, we run the model with $K-1$ partitions for training and 1 partition for validation. At the end, we take an average of the validation results as the final result. The cross validation scheme can be combined with a final testing phase by only dividing the training and validation data into $K$ partitions, while leaving some data points for testing after the cross validation is finished.

\subsubsection{Dynamic Propagation} \label{sec:prop}
In \cite{SDG}, the authors utilize a dynamic propagation scheme to replace the usual message passing propagation. This not only reduces the training time, but also allows for larger neighborhoods to be considered in their model.
Following the idea in the aforementioned paper, the DHGNN model utilizes the 2-hop dynamic propagation whereas the DBGNN model utilizes the mixed dynamic propagation, which are both a variation of the original dynamic propagation method in \cite{SDG}, applied on a heterogeneous graph.

\subsubsection*{2-Hop Dynamic Propagation}
To provide a timely prediction, we aim to apply dynamic propagation scheme, proposed in \cite{SDG}, on graph $G = (V,E)$. This scheme allows the model to recompute the propagation matrix when there is new information without having to retrain the model.
In more detail, each row of the dynamic propagation matrix $P$, or $P(i,:)$, is a stationary distribution of a random walk with restart (RWR) from node $V_i$, which can be expressed as follows. Let $V_i$ be the seed node. Then,
\begin{align} \label{eq:P_def}
   P(i,:)^\top \,=\, \alpha \, M P(i,:)^\top + (1-\alpha) \,\bse_i,
\end{align}
where $\alpha \in [0,1)$ is the probability of moving through an outgoing edge from the current node, and $1-\alpha$ the restart probability, i.e., the probability of jumping back to the seed node. In the equation, $M = A D^{-1}$ is a $n \times n$ column stochastic matrix, which is obtained by multiplying the $n \times n$ adjacency matrix $A$ with the inverse of the $n \times n$ diagonal matrix $D$ whose entries $d_{i}$ are the degrees of node $V_i$ for all $i$. The vector $\bse_i$ denotes the standard basis vector where all its components are zero, except the $i$-th element, which is one.

When there is a change in the graph topology, i.e., insert/delete edge and/or new nodes, each row of the dynamic propagation matrix $P$ needs to be updated. The exact computation of RWR, and consequently $P$, is expensive, however, there are fast and scalable methods which provide good approximation. One such methods uses cumulative power iteration (CPI) \cite{TPA} by pushing out the previous probability distribution score from the changed part to the residual part of the graph, as outlined in \cite{SDG, OSP}. Mathematically, the process can be written as
\begin{align*}
    P(i,:)^\top_{\mathrm{pushout}} &\,=\, \alpha (M' - M) P(i,:)^\top, \hspace{1em} \text{and}\\
    P(i,:)^\top_{\mathrm{new}} &\,=\,  P(i,:)^\top + \sum_{k=0}^\infty \, (\alpha M')^k P(i,:)^\top_{\mathrm{pushout}},
\end{align*}
where $M' = A' D'^{-1}$, $A'$ is the new adjacency matrix and $D'$ is the new diagonal degree matrix after the graph is updated. Note that the convergence and exactness of the above process have been studied in \cite[Section 3.1]{OSP}, which shows that the method does converge and result in the new probability distribution score as $k \rightarrow \infty$.

However, due to the heterogeneous nature of the problem, the unipartite dynamic propagation in the aforementioned paper cannot be applied directly to our bipartite news-author network, and consequently the heterogeneous network. To resolve this issue on the bipartite network, instead of using 1-hop random walk, we apply 2-hop random walk. Thus, $M_2 = M^2 = (AD^{-1})^2$. Since $M$ is a stochastic matrix and the fact that a product of stochastic matrices is also a stochastic matrix, $M_2$ is also a stochastic matrix.

To prove the convergence and exactness of the 2-hop dynamic propagation, we will follow the proof in \cite{OSP} and replace $M$ with $M_2$ as shown in the following theorem.

\begin{theorem}[Convergence and Exactness of 2-Hop Dynamic Propagation Scheme] \label{thm:two_hop}
\hspace{0.5em} Let $G = (V,E)$ be the bipartite news-author graph. Suppose that $G$ is updated to $G' = (V',E')$ such that the 2-hop dynamic propagation scheme is given by
\begin{align*}
    P(i,:)^\top_{\mathrm{pushout}} &\,=\, \alpha (M_2' - M_2) P(i,:)^\top, \hspace{1em} \text{and}\\
    P(i,:)^\top_{\mathrm{new}} &\,=\,  P(i,:)^\top + \sum_{k=0}^\infty \, (\alpha M_2')^k P(i,:)^\top_{\mathrm{pushout}},
\end{align*}
where $P(i,:)$ is the $i$-th row of the dynamic propagation matrix $P$, $\alpha \in [0,1)$ is the probability of jumping to the outgoing edge of the current node and $M_2 = (AD^{-1})^2$.
Then, $\sum_{k=0}^\infty \, (\alpha M_2')^k P(i,:)^\top_{\mathrm{pushout}}$ converges to a constant value as $k \rightarrow \infty$ and $P_{\mathrm{new}}$ is the new propagation matrix of the updated graph $G'$ whose rows are denoted as $P(i,:)_{\mathrm{new}}$.
\end{theorem}

\begin{proof}
We will begin by proving the convergence of the summation term. By expanding the pushout term, we have
\begin{align} \label{eq:conv_proof}
    &\sum_{k=0}^\infty \, (\alpha M_2')^k P(i,:)^\top_{\mathrm{pushout}} 
    \,=\, \sum_{k=0}^\infty \, (\alpha M_2')^k \alpha (M_2' - M_2) P(i,:)^\top \\
    &\,=\,
    \sum_{k=0}^\infty \, (\alpha M_2')^{k+1} P(i,:)^\top - \big(\alpha^{k+1} (M_2')^k M_2 \big) \, P(i,:)^\top. \nonumber
\end{align}
Since $M_2$ and $M_2'$ are stochastic matrices, we have $\|M_2\|_1 = 1$ and $\|M_2'\|_1 = 1$. Moreover, $\|P(i,:)\|_1 = 1$ as $P(i,:)$ is a stationary distribution of $V_i$. Thus, the summation terms converge since $\alpha \rightarrow 0$ as $k \rightarrow \infty$.

We will exhaustively prove that $P(i,:)_{\mathrm{new}}$ is indeed the $i$-th row of the new propagation matrix $P_{\mathrm{new}}$. Firstly, we check the case when there is no update in the graph, which implies that $M_2 = M_2'$ and subsequently, $P(i,:)^\top_{\mathrm{pushout}} = \bszero$. It is clear that in this case, $P(i,:)_{\mathrm{new}}^\top = P(i,:)^\top$. When there is an update in the graph, we begin by expanding the pushout summation term as shown in \eqref{eq:conv_proof}. Then, rewriting vector $P(i,:)^\top$ as $\sum_{k' = 0}^\infty (\alpha M_2)^{k'} P(i,:)^\top_{\mathrm{old}}$, we have
\begin{align*}
    &\sum_{k=0}^\infty \, (\alpha M_2')^k P(i,:)^\top_{\mathrm{pushout}} \\
    & \,=\, \sum_{k=0}^\infty \bigg( (\alpha M_2')^k \alpha (M_2' - M_2) \sum_{k' = 0}^\infty (\alpha M_2)^{k'} P(i,:)^\top_{\mathrm{old}} \bigg) \\
    &\,=\,
    \sum_{k=0}^\infty \bigg( (\alpha M_2')^{k+1} \sum_{k'=0}^\infty (\alpha M_2)^{k'} P(i,:)^\top_{\mathrm{old}} \\
    & \,\,\,\,\,\,\, - (\alpha M_2')^{k} \sum_{k'=0}^\infty (\alpha M_2)^{k' + 1} P(i,:)^\top_{\mathrm{old}} \bigg).
\end{align*}
Since the last summation term in the sum over $k$ can be rewritten as
\begin{align*}
    \sum_{k'=0}^\infty (\alpha M_2)^{k' + 1} P(i,:)^\top_{\mathrm{old}}
    &\,=\,
    -P(i,:)^\top_{\mathrm{old}} +\sum_{k'=0}^\infty (\alpha M_2)^{k'} P(i,:)^\top_{\mathrm{old}},
\end{align*}
and the first term (including the outer sum) can be rewritten as
\begin{align*}
    & \sum_{k=0}^\infty (\alpha M_2')^{k+1} \sum_{k'=0}^\infty (\alpha M_2)^{k'} P(i,:)^\top_{\mathrm{old}} \\
    &\,=\,
    -\sum_{k'=0}^\infty (\alpha M_2)^{k'} P(i,:)^\top_{\mathrm{old}} + \sum_{k=0}^\infty (\alpha M_2')^{k} \sum_{k'=0}^\infty (\alpha M_2)^{k'} P(i,:)^\top_{\mathrm{old}}, 
\end{align*}
we have
\begin{align*}
    & \sum_{k=0}^\infty \, (\alpha M_2')^k P(i,:)^\top_{\mathrm{pushout}} \\
    &\,=\, 
    -\sum_{k'=0}^\infty (\alpha M_2)^{k'} P(i,:)^\top_{\mathrm{old}} + \sum_{k=0}^\infty (\alpha M_2')^{k} P(i,:)^\top_{\mathrm{old}} \\
    &\,=\,
    -P(i,:)^\top + \sum_{k=0}^\infty (\alpha M_2')^{k} P(i,:)^\top_{\mathrm{old}}. 
\end{align*}
Hence, substituting the above equation to the formula for $P(i,:)^\top_{\mathrm{new}}$ yields
\begin{align*}
    P(i,:)^\top_{\mathrm{new}} 
    &\,=\, 
    P(i,:)^\top + \sum_{k=0}^\infty \, (\alpha M_2')^k P(i,:)^\top_{\mathrm{pushout}} \\
    &\,=\,
    \sum_{k=0}^\infty (\alpha M_2')^{k} P(i,:)^\top_{\mathrm{old}},
\end{align*}
which is exactly the stationary distribution of $P(i,:)$ using CPI.
\end{proof}

\subsubsection*{Mixed Dynamic Propagation}
In DHGNN, we take into account three different mappings among news and author entities: news-author, news-news and author-author mappings. Note that the news-author mapping can be independently represented by the 2-hop dynamic propagation derived in the previous section, whereas the homogeneous news-news and author-author mappings can be represented by the original dynamic propagation in \cite{SDG}. Thus, the mixed dynamic propagation is a weighted average of the two-hop dynamic propagation and the original dynamic propagation. To be more precise, given the two-hop propagation matrix $P_{an}$ and the dynamic propagation matrices $P_{nn}$ and $P_{aa}$, the mixed propagation matrix $P$ is defined as 
\begin{align*}
    P &\,=\, \beta_0 P_{an} + \beta_1 P_{nn} + \beta_2 P_{aa},
\end{align*}
where $\beta_0 + \beta_1 + \beta_2 = 1$, and $P_{an}$, $P_{nn}$ and $P_{aa}$ represent the propagation matrices of author-news, news-news and author-author relationships, respectively.

Similar to the 2-hop dynamic propagation scheme, the convergence and exactness of the mixed dynamic propagation scheme can be proved by observing that the weighted sum of stochastic matrices is also a stochastic matrix. Hence, the mixed dynamic propagation matrix is a stochastic matrix. Then, the proof of Theorem~\ref{thm:two_hop} can be applied to the mixed dynamic propagation scheme by replacing the two-hop propagation matrix with the mixed propagation matrix.

\section{Experiments} \label{sec:experiments}
In this section, we will analyze the performance of our model in comparison to existing state-of-the-art agnostic and task-dependent homogeneous and heterogeneous models. 
To compare our models and others, we conducted a grid search for the most optimal hyperparameters for each model listed in Section \ref{sec:models} with 4-fold cross validation.
The metrics observed include the accuracy, precision, recall, AUC and F1 scores of the model on the validation set, and the training time with early stopping. The training is only stopped when the training loss does not improve within 10 consecutive epochs. Note that the quoted training time includes the propagation scheme of the model.

All experiments are conducted using a shared computational cluster called \href{https://doi.org/10.26190/669x-a286}{Katana} supported by Research Technology Services at UNSW Sydney with up to six Tesla V100-SXM2-32GB GPUs and eight CPUs run on CentOS Linux 7. 
\subsection{Data Sets}
For our research on the timely fake news detection with bipartite and heterogeneous news-author networks, there are two suitable data sets that will be used: LIAR \eqref{dataset:liar} and FakeNewsNet \eqref{dataset:fakenewsnet}.

\subsubsection{LIAR} \label{dataset:liar}
\emph{LIAR} contains 12,800 manually labelled short statements from various subjects, sourced from PolitiFact over a 10-year-period from 2007 to 2016\cite{wang2017liar}. It groups the data into training, testing and validation, with each row containing the following information:
\begin{center}
    \{id, statement, label, subject(s), speaker name, speaker profile (job title, US base state,\\ party affiliation), context, statement setting, speaker's total historical score\}.
\end{center}
LIAR is publicly available for fake news detection and can access via this \href{https://www.cs.ucsb.edu/~william/data/liar_dataset.zip}{link}.\\

\subsubsection{FakeNewsNet} \label{dataset:fakenewsnet}
\emph{FakeNewsNet} is a tool which enables its users to collect, analyse and visualize fake news detection and dissemination on social media, namely Twitter. Combining the news from PolitiFact and GossipCop, along with the corresponding tweets and retweets related to the news, the data collected are labelled by experts and journalists. The data is separated based on the source, i.e., Politifact or GossipCop, and whether the label is fake or real. 

The repository is publicly available for fake news detection, fake news propagation and can be accessed \href{https://github.com/KaiDMML/FakeNewsNet}{here}. \\

\begin{table}[t]
  \caption{Statistics of the LIAR and FakeNewsNet Datasets.}
  \small\centering
  \begin{tabular}{lc|c|c|c}
    \toprule[.2em]
    & {} & \multicolumn{1}{c|}{\bfseries Liar} & \multicolumn{1}{c|}{\bfseries Gossipcop} &\multicolumn{1}{c}{\bfseries Politifact} \\
    \midrule[.1em]
    \multirow{5}{0.77cm}{\#entity} 
      & {news} & 12,836 & 17,543 & 656  \\
    \cmidrule(lr){2-5}
      & {author} & 3,318 & 3,848 & 315 \\
    \cmidrule(lr){2-5}
      & {subject} & 144 & 0 & 0 \\
      \cmidrule(lr){2-5}
      & {source} & 5,103 & 1,544 & 106  \\
    \midrule[0.1em]
    
    \multirow{4}{0.8cm}{\#link} 
      & {news-author} & 12,836 & 17,542 & 656 \\
    \cmidrule(lr){2-5}
      & {news-news} & 66,288 & N/A & N/A \\
    \cmidrule(lr){2-5}
      & {author-author} & 18,228 & 86,721 & 2,819 \\

    \bottomrule[.2em]
    \end{tabular}
    \label{tab:dataset}
\end{table}

\subsection{Models} \label{sec:models}
\begin{enumerate}
    \item DBGNN
    
    Our dynamic bipartite model without BERT finetuning in the data extraction module.
    
    \item DBGNN$_{finetuning}$
    
    Our dynamic bipartite model with BERT finetuning in the data extraction module.
    
    \item DHGNN
    
    Our dynamic heterogeneous model without BERT finetuning in the data extraction module.
    
    \item DHGNN$_{finetuning}$
    
    Our dynamic heterogeneous model with BERT finetuning in the data extraction module.
    
    \item GCN
    
    2-layer Vanilla GCN from \cite{GCN} for a homogeneous graph-level news classification problem with Word2Vec to get embedding for each text.
    
    \item GAT
    
    2-layer GAT from \cite{GAT} for a homogeneous graph-level news classification problem with Word2Vec to get embedding for each text.
    
    \item HAN
    
    HAN model from \cite{wang2019heterogeneous} for heterogeneous graph-level news classification problem with author and subject information and TfIdfVectorizer for data extraction.
    
    \item HGAT
    
    HGAT model from \cite{HGAT} for a heterogeneous graph-level news classification problem with author and subject information and TfIdfVectorizer for data extraction.
    
    
    
\end{enumerate}

\subsection{Results}

Table~\ref{tab:result_liar}, ~\ref{tab:result_Gossipcop}, and~\ref{tab:result_Politifact} summary the training time and prediction quality of different models in the experiments on each dataset.

\subsubsection{Effectiveness of BERT Finetuning}
In this section, we will compare the quality of the model which uses pretrained BERT model with the model which uses finetuned BERT model in the data extraction phase. Note that the details about pretrained and finetuned BERT models are discussed in Section \ref{sec:data_extract}.

The 5th and 6th columns and the 7th and 8th columns of each table shows that in general, the finetuning of BERT model results in lower training time. The lower training time is caused by an earlier convergence in the models with BERT finetuning as the resulting feature matrices from the data extraction method were obtained from finetuning BERT.However, in terms of prediction quality, based on the five metrics, i.e., accuracy, precision, recall, AUC and F1 score, the finetuning of BERT model shows similar or lower performance. Due to the millions of parameters in the BERT base model and a significantly smaller benchmarking data sets, BERT finetuning results in overfitting, which leads to lower prediction quality in comparison to the models using a pretrained BERT model.

\subsubsection{Incorporation of Contextual Information in the Models}
In this section, we will discuss the influence of contextual information such as news subject and source to the quality of the models, which can be inferred by comparing the same variant of DBGNN and DHGNN models.

Table~\ref{tab:result_liar}, ~\ref{tab:result_Gossipcop}, and~\ref{tab:result_Politifact} shows that the training in the DHGNN model converges slower than DHGNN. This is due to the heavier sparse matrix computations in the dynamic propagation matrix calculation. This is expected as the number of matrix multiplications is tripled in each step.
On the other hand, we also observe that the predictions made by DHGNN and DBGNN are different depending on the dataset. In LIAR and Gossipcop dataset, DBGNN makes predictions with slightly higher accuracy than DHGNN. However, in case of Politifact dataset, the result is different from the previous result. A possible explanation of this occurrence is the insufficient cleaning and filtering of the contextual information in the data sets. For example, the source column in the LIAR data set is often descriptive and not categorical, and there is not many news in the Gossipcop data set which contains subject information. Therefore, the use of contextual information in our models is not consistent. Further studies are needed to confirm the effect of adding contextual information into the models.

\subsubsection{Comparison on the Training Time of Different Models}
It is clear that our bipartite models have the shortest training time for LIAR, Gossipcop and Politifact data sets. Our heterogeneous model is the second fastest model. Compared to the bipartite model, DHGNN took 23\%, 27\% and 39\% more training time for LIAR, Gossipcop and Politifact dataset, respectively. HGAT trained for more than twice the training duration of the heterogeneous models, followed by HAN models with around 234 seconds for LIAR, 783 seconds for Gossipcop, and 152 seconds for Politifact data sets. Finally, the homogeneous models, namely GCN and GAT, were significantly slower with around 300 seconds, over 1200 seconds, and over 200 seconds training time for LIAR, Gossipcop and Politifact dataset, respectively.

\subsubsection{Comparison on the Prediction Quality of Different Models}
For LIAR data set in Table\ref{tab:result_liar}, most of the models in the experiment produced similar results, with around 72.3\% accuracy and precision, and nearly 100\% recall. This is due to the small size of the news content and data set. HGAT and HAN, however, performed slightly worse than the other models with 70.17\% and 71.94\% accuracy. 
On the other hand, HGAT performed best when predicting the Gossipcop data set, with 82.64\% accuracy, 82.18\% precision and 97.42\% recall. This result is considerably higher than other models in the experiment, including our bipartite and heterogeneous models. Our models without BERT finetuning perform second best with approximately 80\% accuracy, 79\% precision, and almost 100\% recall. GCN and GAT produce similar results with around 76\% accuracy and precision, and 100\% recall. HAN, a heterogeneous model, showed better accuracy than GAT and GCN with an accuracy of 76.58\%.

As DBGNN and DHGNN use GCN as classifier in this experiment, we can see that by using the dynamic propagation scheme, we are able to reduce the training time of a vanilla GCN significantly. In fact, for LIAR data set, the training time of DBGNN and DHGNN are $\frac{1}{20}$ and $\frac{3}{50}$ of GCN training time, respectively. Similar observations can be made when trained with the FakeNewsNet data set.
Moreover, when compared with vanilla GCN, we are able to retain the prediction quality, as evident by our results in Table~\ref{tab:result_liar}, \ref{tab:result_Gossipcop} and \ref{tab:result_Politifact}.

We are aware that when trained and tested with the Gossipcop and Politifact data set, HGAT performs best in regards to both training time and prediction quality. We suspected that the results of DBGNN and DHGNN were worse than our expectation due to some overfitting as we use GCN as our classifier, although further investigation is needed to ascertain the actual cause.

\begin{table*}[t]
  \caption{Training Time and Prediction Quality of Different models on \textit{LIAR} Dataset.}
  \normalsize\centering
  \begin{tabular}{c|cccccccc}

    \toprule[.2em]
    {\bfseries Metric }& {GAT} & {GCN} & {HAN} & {HGAT} &{DBGNN} &{DBGNN$_{f}$} &{DHGNN} &{DHGNN$_{f}$}  \\
    \midrule[.1em]

      {Accuracy} 
      & 0.7226 & 0.7228 & 0.7194 & 0.7017 & 0.7227 & 0.7228 & 0.7226 & 0.7227 \\
      
    \cmidrule(lr){2-9}
      {Precision}
      & 0.7226 & 0.7228 & 0.7192 & 0.7120 & 0.7227 & 0.7232 & 0.7234 & 0.7230 \\
    \cmidrule(lr){2-9}
      {Recall} 
      & 1.0000 & 0.9998 & 0.9997 & 0.9847 & 0.9984 & 0.9972 & 0.9969 & 0.9978 \\
      \cmidrule(lr){2-9}
      {F1-score}
      & 0.8390 & 0.8390 & 0.8198 & 0.8147 & 0.8386 & 0.8390 & 0.8383 & 0.8393 \\  
      \cmidrule(lr){2-9}
      {Time(s)} 
      & $318.75\pm24.56$ & $282.09\pm52.87$  & $183.77\pm54.13$  & $41.28\pm10.28$  & $12.93\pm5.17$  & $12.28\pm4.56$ & $15.92 \pm4.60$ & $15.48\pm4.32$ \\  
    
    \bottomrule[.2em]
  \end{tabular}
  \label{tab:result_liar}
\end{table*}

\begin{table*}[t]
  \caption{Training Time and Prediction Quality of Different models on \textit{Gossipcop} Dataset.}
  \normalsize\centering
  \begin{tabular}{c|cccccccc}

    \toprule[.2em]
    {\bfseries Metric }& {GAT} & {GCN} & {HAN} & {HGAT} &{DBGNN} &{DBGNN$_{f}$}&{DHGNN}&{DHGNN$_{f}$} \\
    \midrule[.1em]

      {Accuracy} 
      & 0.7608 & 0.7549 & 0.7658 & 0.8264 & 0.7985 & 0.7875 & 0.7978 & 0.7874\\
    \cmidrule(lr){2-9}
      {Precision}
      & 0.7608 & 0.7549 & 0.7662 & 0.8218 & 0.7945 & 0.7876 & 0.7990 & 0.7880\\
    \cmidrule(lr){2-9}
      {Recall} 
      & 1.0000 & 1.0000 & 0.9889 & 0.9742 & 0.9962 & 0.9964 & 0.9967 & 0.9986\\
      \cmidrule(lr){2-9}
      {F1-score}
      & 0.8550 & 0.8550 & 0.8613 & 0.8932 & 0.8702& 0.8705 & 0.8725 & 0.8712 \\
      \cmidrule(lr){2-9}
      {Time(s)} 
      & $1223.19\pm100.05$ & $1098.37\pm158.08$ & $783.67\pm4.56$ & $68.39\pm24.51$ & $21.67\pm5.90$ & $19.54\pm4.56$ & $27.58 \pm5.16$ & $23.71\pm4.20$ \\
    
    \bottomrule[.2em]
  \end{tabular}
  \label{tab:result_Gossipcop}
\end{table*}

\begin{table*}[t]
  \caption{Training Time and Prediction Quality of Different models on \textit{Politifact} Dataset.}
  \normalsize\centering
  \begin{tabular}{c|cccccccc}

    \toprule[.2em]
    {\bfseries Metric }& {GAT} & {GCN} & {HAN} & {HGAT} &{DBGNN} &{DBGNN$_{f}$}&{DHGNN}&{DHGNN$_{f}$} \\
    \midrule[.1em]

      {Accuracy} 
      & 0.7405 & 0.7405 & 0.7481 & 0.7710 & 0.7557 & 0.7557 & 0.7634 & 0.7634\\
    \cmidrule(lr){2-9}
      {Precision}
      & 0.7833 & 0.7881 & 0.7949 & 0.8051 & 0.7966 & 0.8017 & 0.8087 & 0.8034\\
    \cmidrule(lr){2-9}
      {Recall} 
      & 0.9216 & 0.9118 & 0.9118 & 0.9314 & 0.9216 & 0.9118 & 0.9118 & 0.9216\\
      \cmidrule(lr){2-9}
      {F1-score}
      & 0.8500 & 0.8551 & 0.8551 & 0.8649 & 0.8594 & 0.8551 & 0.8649 & 0.8651\\
      \cmidrule(lr){2-9}
      {Time(s)} 
      & $201.63\pm9.70$ & $194.88\pm11.05$ & $152.89\pm4.56$ & $19.41\pm4.23$ & $5.99\pm1.40$ & $5.04\pm1.51$ & $8.35\pm2.02$ & $7.90\pm1.34$ \\
    
    \bottomrule[.2em]
  \end{tabular}
  \label{tab:result_Politifact}
\end{table*}

\section{Conclusion and Future Outlook} \label{sec:conclusion}
The field of deep learning has brought tremendous contribution to the society with applications in the computer vision, natural language processing, and many others. The emerging branch of GNN has slowly shown some promising results in tackling various problems such as automating traffic systems, analysing social networks, and discerning mis(dis)information in the social media.



\bibliographystyle{ACM-Reference-Format}
\bibliography{ref}

\end{document}